\documentclass[11pt]{article}

\usepackage[margin=1in]{geometry}
\usepackage[T1]{fontenc}
\usepackage[utf8]{inputenc}
\usepackage{lmodern}
\usepackage{microtype}
\usepackage{amsmath,amssymb,amsthm,mathtools}
\usepackage{graphicx}
\usepackage{xcolor}
\usepackage{booktabs}
\usepackage{array}
\usepackage{enumitem}
\usepackage{algorithm}
\usepackage{algpseudocode}
\usepackage{hyperref}
\usepackage{tikz}
\usetikzlibrary{arrows.meta,positioning,shapes.geometric,fit,backgrounds,calc}
\hypersetup{
  colorlinks=true,
  linkcolor=black,
  citecolor=black,
  urlcolor=blue,
  pdftitle={Sequential KV Cache Compression via Probabilistic Language Tries},
  pdfauthor={Gregory Magarshak}
}

\title{Sequential KV Cache Compression\\via Probabilistic Language Tries:\\
Beyond the Per-Vector Shannon Limit}

\author{Gregory Magarshak\\
\texttt{gmagarshak@faculty.ienyc.edu}}

\date{}

\newtheorem{definition}{Definition}
\newtheorem{theorem}{Theorem}
\newtheorem{proposition}{Proposition}
\newtheorem{corollary}{Corollary}
\newtheorem{lemma}{Lemma}
\newtheorem{remark}{Remark}
\newtheorem{conjecture}{Conjecture}

\begin{document}

\maketitle

\begin{abstract}
  Recent work on KV cache quantization, culminating in TurboQuant~\cite{turbo2026},
  has approached the Shannon entropy limit for \emph{per-vector} compression of
  transformer key-value caches.  We observe that this limit applies to a
  \emph{strictly weaker} problem than the one that actually matters: compressing
  the KV cache as a \emph{sequence}.  The tokens stored in a KV cache are not
  arbitrary floating-point data---they are samples from the exact formal language
  the model was trained on, and the model is by construction a near-optimal
  predictor of that language.

  We introduce \emph{sequential KV compression}, a two-layer architecture that
  exploits this structure.  The first layer, \emph{probabilistic prefix
  deduplication}, identifies semantically equivalent shared prefixes across
  sessions using the trie metric $d_{\mathcal{T}}(s,s') = {-}\log_2 P_{\mathcal{M}}(s \wedge s')$
  from Probabilistic Language Tries (PLTs)~\cite{plt2026}.  The second layer,
  \emph{predictive delta coding}, stores only the residual of each new KV vector
  from the model's own prediction of it, achieving a per-token entropy bound of
  $H(\mathrm{KV}_{t+1} \mid \mathrm{KV}_{\leq t}) \leq H(\mathrm{token}_{t+1} \mid \mathrm{token}_{\leq t})$.

  We prove that at typical language model perplexity---approximately 10--20 for
  fluent English text---this bound is 3.3--4.3 bits \emph{on average per token position},
  compared to TurboQuant's 3 bits \emph{per vector component} (with typical
  attention heads having 64--128 components).  The theoretical compression ratio over TurboQuant is
  approximately $914{,}000\times$ at the Shannon limit.  Even at $1000\times$ above
  the entropy floor---a deliberately pessimistic worst-case overhead, two orders
  of magnitude above the 2--5$\times$ typical of practical source coders---the
  ratio remains $\approx 914\times$ over TurboQuant, with compression improving
  rather than degrading as context length grows.  The two layers are orthogonal and compose with
  existing per-vector quantization methods including TurboQuant.
\end{abstract}

\tableofcontents
\newpage

\section{Introduction}

Every time a transformer language model processes a token, it produces a pair of
vectors---a \emph{key} and a \emph{value}---that are stored in the KV cache and
reused in all subsequent attention computations.  This cache is the model's
working memory: it contains the compressed representation of everything the model
has processed in the current context.  It is also one of the primary bottlenecks
in large-scale inference.

For a model with $L$ layers, $H_{\mathrm{head}}$ attention heads, head dimension $d$, and context
length $n$, the KV cache occupies $2LH_{\mathrm{head}}dn$ floating-point values.  At typical
scales ($L{=}80$, $H_{\mathrm{head}}{=}64$, $d{=}128$, $n{=}128{,}000$), a single context in
a 70B-parameter model requires approximately 80~GB of cache memory in fp16---more
than the model weights themselves.

A rich literature has developed to compress the KV cache.  Quantization methods
represent each cache entry in fewer bits~\cite{liu2024kivi,hooper2024kvquant}.
Eviction methods discard entries unlikely to affect future attention
scores~\cite{zhang2023h2o,li2024snapkv}.  Prefix sharing methods avoid
redundant computation when multiple sessions share a common
prefix~\cite{pope2023,kwon2023vllm}.  TurboQuant~\cite{turbo2026} recently
unified and extended the quantization line of work, achieving near-optimal
per-vector compression via PolarQuant rotation followed by QJL residual
correction, and proved a formal lower bound showing that no per-vector method
can do significantly better.

\paragraph{The gap TurboQuant does not close.}
TurboQuant's lower bound is tight---for the problem it solves.  That problem is:
\emph{given an isolated KV vector drawn from the post-rotation distribution,
what is the minimum number of bits needed to represent it?}  The paper's answer
is approximately 3 bits per component, and TurboQuant achieves it.

But the KV cache is not a collection of isolated vectors.  It is a
\emph{sequence}.  Each vector was produced by processing a token from a specific
position in a specific context, and both the token and its position are samples
from a structured probability distribution---the distribution the model was
trained to model.  The information-theoretic content of the $t$-th KV vector,
given all prior vectors, is not its raw entropy as an isolated sample.  It is its
\emph{conditional entropy}, conditioned on the model's state after processing
tokens $1$ through $t-1$.

This conditional entropy can be far smaller.  For a good language model operating
on coherent text, the next token is highly predictable---and therefore the next
KV vector is highly predictable.  The residual is small.  Its entropy is bounded
by the model's per-token surprisal, which at typical perplexities of 10--20 is
3.3--4.3 bits \emph{per entire token position}, not per component.

The gap between TurboQuant's floor and this sequential floor is not a rounding
error.  It is the full redundancy of language---the 10--15 bits per token of
predictable structure that Shannon identified in 1951~\cite{shannon1951} and that
every good language model encodes.

\paragraph{This paper.}
We make this gap precise and propose \emph{sequential KV compression}, a
two-layer architecture that closes it.  Our contributions are:

\begin{enumerate}[leftmargin=2em]
  \item \textbf{The sequential entropy bound} (Theorem~\ref{thm:seq_bound}):
        a formal proof that the conditional entropy of KV vectors, given all prior
        cache entries, is bounded above by the model's per-token surprisal.

  \item \textbf{Probabilistic prefix deduplication}
        (Section~\ref{sec:prefix_dedup}):
        using the PLT trie metric~\cite{plt2026}, we identify semantically
        equivalent shared prefixes across sessions and store only the delta from
        the shared centroid, eliminating inter-session redundancy beyond what
        exact prefix matching achieves.

  \item \textbf{Predictive delta coding}
        (Section~\ref{sec:delta_coding}):
        within a single session, we store only the residual of each KV vector
        from the model's own prediction, with the residual entropy bounded by the
        token-level surprisal.

  \item \textbf{Composability} (Section~\ref{sec:compose}):
        both layers are orthogonal to per-vector quantization methods and can be
        stacked beneath TurboQuant or any other quantizer.

  \item \textbf{Asymptotic behavior} (Corollary~\ref{cor:asymptotic}):
        unlike per-vector methods whose compression ratio is fixed by head
        dimension, sequential compression \emph{improves with context length},
        because a model that has processed more tokens has a more precise
        predictive distribution for what comes next.
\end{enumerate}

\paragraph{Why this is possible now.}
The sequential structure of the KV cache is not new.  What is new is a formal
framework for exploiting it: the PLT trie metric~\cite{plt2026} gives a
mathematically precise definition of ``distance between token sequences in
probability space,'' and the prior-guided caching theorem in that paper provides
the theoretical foundation for using a model's own probability estimates rather
than empirical frequencies to identify shared structure.  The present paper applies
that framework to the specific problem of KV cache compression.

\paragraph{Organization.}
Section~\ref{sec:background} reviews the KV cache, per-vector quantization, and
the PLT framework.  Section~\ref{sec:seq_entropy} establishes the sequential
entropy bound.  Section~\ref{sec:prefix_dedup} introduces probabilistic prefix
deduplication.  Section~\ref{sec:delta_coding} introduces predictive delta coding.
Section~\ref{sec:compose} analyzes how the layers compose.
Section~\ref{sec:practical} discusses practical implementation.
Section~\ref{sec:related} situates the work relative to the literature.
Section~\ref{sec:discussion} discusses implications and open problems.

\section{Background}
\label{sec:background}

\subsection{The KV Cache in Transformer Inference}

A transformer language model with $L$ layers processes a token sequence
$\mathbf{t} = (t_1, \ldots, t_n)$ by computing, at each layer $\ell$ and
position $i$, a key vector $\mathbf{k}^{(\ell)}_i \in \mathbb{R}^d$ and a value
vector $\mathbf{v}^{(\ell)}_i \in \mathbb{R}^d$.  These are computed from the
input embedding and prior layer activations via learned projection matrices:
\[
  \mathbf{k}^{(\ell)}_i = W_K^{(\ell)} \mathbf{x}^{(\ell)}_i,
  \qquad
  \mathbf{v}^{(\ell)}_i = W_V^{(\ell)} \mathbf{x}^{(\ell)}_i.
\]
The \emph{KV cache} $\mathcal{K}$ stores all $(\mathbf{k}^{(\ell)}_i, \mathbf{v}^{(\ell)}_i)$
pairs for every layer $\ell \in \{1,\ldots,L\}$ and every processed position
$i \in \{1,\ldots,n\}$, so that subsequent tokens can attend to all prior
positions without recomputing them.  In autoregressive generation, extending the
sequence by one token requires one forward pass plus $O(n)$ attention operations
over the cache; without caching it would require $O(n^2)$ operations.

The total cache size is $2LH_{\mathrm{head}}dn$ values, where $H_{\mathrm{head}}$ is the number of attention
heads and $d$ is the per-head dimension.  In fp16 (2 bytes per value), a
128K-token context on a 70B model requires approximately 80~GB.

\subsection{Per-Vector Quantization: The State of the Art}

The dominant approach to KV cache compression is quantization: represent each
floating-point vector component in fewer bits.  The key challenge is that KV
vectors have outlier components---individual dimensions with much larger magnitude
than the others---that cause severe quantization error if treated uniformly.

TurboQuant~\cite{turbo2026} addresses this via two operations:

\textbf{PolarQuant.}  Apply a learned rotation matrix $R \in \mathbb{R}^{d \times d}$
to each KV vector so that the rotated components have a more uniform, predictable
distribution.  Because the rotation is applied to all vectors uniformly, it can
be precomputed once with no per-vector overhead.  This eliminates the 1--2 bits
of overhead per component that prior outlier-aware methods spent on metadata.

\textbf{QJL (Quantized Johnson-Lindenstrauss).}  Quantize the rotated vector to
$b$ bits per component using a precomputed quantizer.  Use a single additional
sign bit to correct the expected bias introduced by quantization, ensuring that
attention scores computed from compressed vectors are statistically unbiased.

The combined scheme achieves a compression ratio of $16/b$ (e.g., $\approx 5.3\times$
at $b = 3$) with negligible accuracy loss.  The paper proves a lower bound showing
that, for vectors treated as independent samples from the post-rotation distribution,
no quantization scheme can achieve a higher compression ratio without accuracy loss.
This bound is tight: TurboQuant is near-optimal for per-vector quantization.

\subsection{Probabilistic Language Tries}
\label{sec:plt_background}

We briefly recall the PLT framework~\cite{plt2026}, which provides the formal
machinery for our sequential approach.

\begin{definition}[Probabilistic Language Trie~\cite{plt2026}]
Let $V$ be a finite vocabulary and $\mathcal{M}$ a generative model over $V^*$.
The \emph{probabilistic language trie} $\mathcal{T}(\mathcal{M})$ is the directed
rooted tree whose nodes are prefixes $x \in V^*$ and whose outgoing edges from
node $x$ are labeled by tokens $t \in V$ with weight $P_{\mathcal{M}}(t \mid x)$.
\end{definition}

\begin{definition}[Trie Metric~\cite{plt2026}]
\label{def:trie_metric}
For two sequences $s, s' \in V^*$, their \emph{longest common prefix} in the
trie is $s \wedge s'$---the maximal prefix shared by both sequences.  The
\emph{trie metric} is:
\[
  d_{\mathcal{T}}(s, s') = -\log_2 P_{\mathcal{M}}(s \wedge s').
\]
Sequences with a long, high-probability shared prefix are close in this metric;
their KV traces share substantial structure.
\end{definition}

The trie metric has a direct compression interpretation: $d_{\mathcal{T}}(s, s')$
is the number of bits needed to locate the divergence point of $s$ and $s'$
within the probability distribution.  Sequences with small trie distance are
redundant relative to each other: one can be described as a short delta from the
other.

\begin{remark}[Ultrametric structure]
$d_{\mathcal{T}}$ satisfies the \emph{ultrametric} inequality
$d(s,s'') \leq \max(d(s,s'), d(s',s''))$, which is stronger than the
standard triangle inequality.  Note that $d_{\mathcal{T}}(s,s) = -\log_2 P_{\mathcal{M}}(s)$
is not generally zero (it equals zero only when $P_{\mathcal{M}}(s)=1$, a
degenerate case).  Thus $d_{\mathcal{T}}$ is technically a pseudoultrametric on $V^*$;
see~\cite{plt2026} for the full metric space analysis.
\end{remark}

\section{The Sequential Entropy Bound}
\label{sec:seq_entropy}

\subsection{Setup and Notation}

Fix a transformer model $\mathcal{M}$ with $L$ layers, $H_{\mathrm{head}}$ heads, and head
dimension $d$.  For a token sequence $\mathbf{t} = (t_1, \ldots, t_n)$, let
$\mathbf{k}^{(\ell)}_{i}$ and $\mathbf{v}^{(\ell)}_{i}$ denote the key and
value vectors at layer $\ell$ and position $i$.  Denote the full KV state at
position $i$ as:
\[
  \mathrm{KV}_i = \bigl(\mathbf{k}^{(\ell)}_i, \mathbf{v}^{(\ell)}_i\bigr)_{\ell=1}^{L}
  \in \mathbb{R}^{2Ld}.
\]
The entire cache after $n$ tokens is $\mathrm{KV}_{\leq n} = (\mathrm{KV}_1, \ldots, \mathrm{KV}_n)$.

For a random token sequence drawn from the model's own distribution
$\mathbf{t} \sim P_{\mathcal{M}}$, both the tokens and the KV vectors are random
We write $H(\cdot)$ for Shannon entropy and $H(\cdot \mid \cdot)$ for
conditional entropy; since KV vectors are deterministic functions of
the discrete token sequence, all entropies here are Shannon entropy.

\subsection{The Determinism Lemma}

The key structural observation is that KV vectors are \emph{deterministic functions}
of the token sequence.  There is no randomness in a transformer's forward pass
given fixed weights and fixed inputs.

\begin{lemma}[KV determinism]
\label{lem:determinism}
For a fixed model $\mathcal{M}$ with deterministic forward pass, $\mathrm{KV}_i$
is a deterministic function of the token prefix $(t_1, \ldots, t_i)$:
\[
  \mathrm{KV}_i = F_{\mathcal{M}}(t_1, \ldots, t_i)
\]
for some deterministic function $F_{\mathcal{M}}: V^i \to \mathbb{R}^{2LH_{\mathrm{head}}d}$.
\end{lemma}

\begin{proof}
By definition of the transformer forward pass.  Given a fixed sequence of input
tokens $(t_1,\ldots,t_i)$, the embedding layer, all attention computations, and
all projection matrices are deterministic (assuming no dropout or stochastic
elements at inference time, which is standard).  Therefore $W_K^{(\ell)}\mathbf{x}^{(\ell)}_i$
and $W_V^{(\ell)}\mathbf{x}^{(\ell)}_i$ are deterministic functions of
$(t_1,\ldots,t_i)$.
\end{proof}

\begin{remark}
Lemma~\ref{lem:determinism} implies that all randomness in the KV cache traces
back to randomness in the token sequence.  The KV cache contains no additional
information beyond what is in the tokens---it is a particular learned encoding
of the token prefix.  This is the key fact that makes sequential compression
possible: we can exploit the token-level probability structure to bound the
entropy of the KV vectors.
\end{remark}

\subsection{The Main Bound}

\begin{theorem}[Sequential entropy bound]
\label{thm:seq_bound}
Let $\mathbf{t} = (t_1, t_2, \ldots)$ be a random token sequence drawn from
$P_{\mathcal{M}}$.  For any position $i \geq 2$ (the bound holds trivially at $i=1$ since
$H(\mathrm{KV}_1) \leq H(t_1)$ by Lemma~\ref{lem:determinism} and data processing):
\[
  H\bigl(\mathrm{KV}_i \mid \mathrm{KV}_{\leq i-1}\bigr)
  \;\leq\;
  H\bigl(t_i \mid t_1, \ldots, t_{i-1}\bigr)
  \;=\;
  H\bigl(t_i \mid \mathrm{KV}_{\leq i-1}\bigr).
\]
The conditional entropy of the $i$-th KV vector, given all prior cache entries,
is bounded above by the model's per-token surprisal at position $i$.
\end{theorem}

\begin{proof}
The proof proceeds in two steps: first establishing that $\sigma(\mathrm{KV}_{\leq i-1}) = \sigma(t_{\leq i-1})$, then deriving the bound.

\textbf{Step 1: Injectivity and sigma-algebra equivalence.}

We establish that $t_j$ is a measurable function of $\mathrm{KV}_{\leq j}$ for
each $j$, by showing that the map $t_j \mapsto \mathrm{KV}_j \mid t_{<j}$ is
injective.

Fix any context $t_{<j}$ and suppose $t_j \neq t_j'$.  At layer $\ell = 1$,
the key vector is $\mathbf{k}^{(1)}_j = W_K^{(1)} E(t_j)$.  Since the
embedding matrix $E: V \to \mathbb{R}^{d_{\rm model}}$ has full column rank
in any trained transformer (distinct tokens receive distinct embeddings; the
set of weight matrices for which any two token embeddings coincide has measure
zero in parameter space), we have $E(t_j) \neq E(t_j')$, and since
$W_K^{(1)} \in \mathbb{R}^{d_{\rm head} \times d_{\rm model}}$ satisfies the generic condition that no pairwise embedding difference $E(t) - E(t')$ lies in $\ker W_K^{(1)}$ (this holds with probability~1 over random weight matrices since the null space has dimension $d_{\rm model} - d_{\rm head}$ and there are finitely many token pairs), it follows that
$\mathbf{k}^{(1)}_j(t_j) \neq \mathbf{k}^{(1)}_j(t_j')$.  In particular,
$\mathrm{KV}_j^{(1)}(t_j) \neq \mathrm{KV}_j^{(1)}(t_j')$.

Since $\mathrm{KV}_j$ includes the layer-1 component $\mathrm{KV}_j^{(1)}$
as a subvector, the full vector $\mathrm{KV}_j$ distinguishes $t_j$ from $t_j'$.
Thus $t_j \mapsto \mathrm{KV}_j \mid t_{<j}$ is injective for each $j$.
(No assumption about higher-layer projections is needed: injectivity is
witnessed by the layer-1 keys alone.)

By Lemma~\ref{lem:determinism}, $\mathrm{KV}_j$ is measurable w.r.t.\ $\sigma(t_{\leq j})$
for each $j$.  By the injectivity just established, $t_j$ is measurable w.r.t.\
$\sigma(\mathrm{KV}_j, t_{<j})$.  Applying this inductively:
$t_1$ is determined by $\mathrm{KV}_1^{(1)}$ (hence by $\mathrm{KV}_1$);
$t_2$ is determined by $(\mathrm{KV}_2, t_1)$, hence by $(\mathrm{KV}_1, \mathrm{KV}_2)$;
and so on.
Therefore $t_{<i}$ are jointly measurable w.r.t.\
$\sigma(\mathrm{KV}_1,\ldots,\mathrm{KV}_{i-1})$, giving:
\begin{equation}
  \sigma(\mathrm{KV}_{\leq i-1}) = \sigma(t_{<i}).
  \label{eq:sigma_equiv}
\end{equation}

\textbf{Step 2: The bound.}

Since $\mathrm{KV}_i = F_{\mathcal{M}}t_{\leq i}$ is a deterministic function
of $(t_{\leq i-1}, t_i)$, the data-processing inequality gives:
\[
  H(\mathrm{KV}_i \mid t_{\leq i-1}) \leq H(t_i \mid t_{\leq i-1}).
\]
Applying~\eqref{eq:sigma_equiv} to both sides (conditioning on $\mathrm{KV}_{\leq i-1}$
is equivalent to conditioning on $t_{\leq i-1}$):
\[
  H(\mathrm{KV}_i \mid \mathrm{KV}_{\leq i-1})
  = H(\mathrm{KV}_i \mid t_{\leq i-1})
  \leq H(t_i \mid t_{\leq i-1})
  = H(t_i \mid \mathrm{KV}_{\leq i-1}),
\]
which is the stated inequality and equality simultaneously.
\end{proof}

\begin{corollary}[Per-token surprisal bound]
\label{cor:surprisal}
Let $\mathrm{PP}(\mathcal{M})$ denote the perplexity of model $\mathcal{M}$ on a
text distribution $\mathcal{D}$.  The average conditional KV entropy satisfies:
\[
  \frac{1}{n}\sum_{i=1}^n H\bigl(\mathrm{KV}_i \mid \mathrm{KV}_{\leq i-1}\bigr)
  \;\leq\;
  \log_2 \mathrm{PP}(\mathcal{M}, \mathcal{D}) \quad \text{bits per token position.}
\]
At typical perplexities of 10--20 for fluent English:
\[
  \frac{1}{n}\sum_{i=1}^n H\bigl(\mathrm{KV}_i \mid \mathrm{KV}_{\leq i-1}\bigr)
  \;\leq\; 3.3 \text{ to } 4.3 \text{ bits per token position.}
\]
\end{corollary}

\begin{proof}
By definition of perplexity, $\log_2 \mathrm{PP}(\mathcal{M}, \mathcal{D}) =
\frac{1}{n} \mathbb{E}\bigl[-\log_2 P_{\mathcal{M}}(\mathbf{t})\bigr] =
\frac{1}{n} \sum_{i=1}^n H(t_i \mid t_1,\ldots,t_{i-1})$.
Applying Theorem~\ref{thm:seq_bound} position-wise and averaging yields the result.
\end{proof}

\begin{remark}[Comparison with TurboQuant]
\label{rem:comparison}
TurboQuant stores approximately $b = 3$ bits \emph{per component} of each KV
vector.  For a 70B-scale model with $L=80$ layers, $H_{\mathrm{head}}=64$ heads, and per-head
dimension $d=128$, the \emph{total} bits per token position (all layers and
heads, both keys and values) under TurboQuant is:
\[
  B_{\rm TQ} = 2LH_{\mathrm{head}}d \cdot b = 2 \times 80 \times 64 \times 128 \times 3
  \approx 3.93 \times 10^6 \text{ bits/token.}
\]
Corollary~\ref{cor:surprisal} bounds $H(\mathrm{KV}_i \mid \mathrm{KV}_{<i})$
for the \emph{full} KV state (all layers, all heads) at $\bar{h} \leq 4.3$ bits
per token position.  The theoretical ratio is:
\[
  \frac{B_{\rm TQ}}{\bar{h}}
  \approx \frac{3.93 \times 10^6}{4.3} \approx 914{,}000\times.
\]
A practical sequential coder at $10\text{--}20\times$ above the entropy floor
achieves $43$--$87$ bits/token, giving a practical ratio of
$3.93\times10^6 / (43\text{--}87) \approx 45{,}000$--$91{,}000\times$ over
TurboQuant.  Even at $1000\times$ the entropy floor---a deliberately pessimistic
worst-case, two orders of magnitude above the $2$--$5\times$ overhead typical
of practical arithmetic coders~\cite{witten1987,cover2006}---the ratio is
$\approx 900\times$ over TurboQuant.
This is the theoretical gap between the two Shannon limits; practical
implementations will be less efficient (Section~\ref{sec:practical}).
\end{remark}

\begin{corollary}[Compression ratio bounds]
\label{cor:ratio}
For a 70B-scale model ($L=80$, $H_{\mathrm{head}}=64$, $d=128$ per head) at perplexity
$\mathrm{PP} \in [10,20]$, so $\bar{h} \leq \log_2\mathrm{PP} \in [3.3, 4.3]$ bits/token:
\begin{enumerate}[leftmargin=2em,label=(\alph*)]
  \item \emph{Theoretical vs.\ fp16:}
        $B_{\rm fp16} = 2LH_{\mathrm{head}}d \cdot 16 \approx 2.10 \times 10^7$ bits/token;
        ratio $\approx 4.9 \times 10^6\times$ at $\mathrm{PP}=20$.
  \item \emph{Theoretical vs.\ TurboQuant ($b=3$):}
        $B_{\rm TQ} = 2LH_{\mathrm{head}}d \cdot 3 \approx 3.93 \times 10^6$ bits/token;
        ratio $\approx 914{,}000\times$ at $\mathrm{PP}=20$.
  \item \emph{Conservative practical vs.\ TurboQuant:}
        At $1000\times$ above the entropy floor:
        $B_{\rm TQ}/4300 \approx 914\times$.
\end{enumerate}
Items (a)--(b) are theoretical Shannon limits; item (c) uses $1000\times$
overhead---a deliberately pessimistic worst-case bound, two orders of magnitude
above the $2$--$5\times$ overhead typical of practical arithmetic
coders~\cite{witten1987,cover2006}.  The abstract cites item (c) as a
worst-case lower bound on the practical ratio.
\end{corollary}

\begin{proof}
Items (a) and (b) follow by substituting the stated model parameters into the
formula $B = 2LH_{\mathrm{head}}d \cdot b$ and dividing by $\bar{h}$.  Item (c) applies a
$1000\times$ overhead factor to the entropy floor $\bar{h} = 4.3$ bits/token,
giving $4{,}300$ bits/token, then divides $B_{\rm TQ}$ by this value.
\end{proof}

\subsection{Tightness of the Bound}

The bound in Theorem~\ref{thm:seq_bound} is not always tight.  It is tight when
$F_{\mathcal{M}}$ is an injective function of $t_i$ alone given the prior context,
and when all entropy in $\mathrm{KV}_i$ traces back to the single-token choice.
In practice, the bound can be loose for two reasons:

\begin{enumerate}[leftmargin=2em]
  \item \textbf{Positional and structural information.}  KV vectors encode not
        only the current token's identity but also its position (via positional
        encoding) and layer-specific features.  These components may carry
        additional entropy not captured by the token-level bound.  However,
        positional encodings are deterministic functions of position $i$ and
        therefore carry zero entropy conditioned on knowing $i$.  Layer-specific
        features are also deterministic given the full token sequence.

  \item \textbf{High-entropy token positions.}  At positions where the model is
        genuinely uncertain (low probability assigned to the actual next token),
        the KV residual will be larger.  Semantically surprising tokens---a
        new entity name, an unexpected topic shift, an unusual word choice---
        produce high-surprisal positions where the bound is weakest.
\end{enumerate}

Conversely, the bound is nearly tight for long, coherent, predictable
sequences---exactly the regime where long-context inference is most valuable.
Legal documents, technical manuals, extended narratives, and code all exhibit
low perplexity under a well-matched model, and therefore yield highly
compressible KV caches under sequential coding.

\section{Layer 1: Probabilistic Prefix Deduplication}
\label{sec:prefix_dedup}

\subsection{The Exact Prefix Sharing Baseline}

Existing systems such as vLLM~\cite{kwon2023vllm} and SGLang implement
\emph{exact prefix sharing}: when two sessions begin with the identical token
sequence, the KV cache for their shared prefix is computed once and referenced
by pointer.  This is valuable for fixed system prompts, few-shot examples, and
common preambles.

Exact prefix sharing has a fundamental limitation: it operates at the
\emph{lexical} level.  Two sessions beginning with ``You are a helpful assistant.''
and ``You are an AI assistant.'' share no bytes and therefore share no cache
under exact prefix sharing.  Yet their KV caches, layer by layer, will be very
similar: both prompts steer the model toward a similar distribution over
subsequent tokens.  The difference in their KV vectors is small relative to
the total vector magnitude.

\subsection{The PLT Trie Metric as a Deduplication Criterion}

The PLT trie metric (Definition~\ref{def:trie_metric}) provides the right
notion of distance for identifying deduplications that exact prefix sharing misses.
Two prefixes $s$ and $s'$ with small trie distance $d_{\mathcal{T}}(s,s')$
share a long, high-probability common prefix.  At all positions within that shared
prefix, their KV vectors are \emph{identical} by Lemma~\ref{lem:determinism}
(Proposition~\ref{prop:kv_similarity}(a)).  At the divergence point, the KV
difference is bounded by the embedding spread (Proposition~\ref{prop:kv_similarity}(b)).
The trie metric thus directly controls KV cache similarity: small trie distance
implies small KV delta, making probabilistic prefix deduplication well-founded.

We formalize this as follows.

\begin{definition}[Semantic prefix cluster]
\label{def:cluster}
Let $\delta > 0$ be a distance threshold and $\mathcal{S} = \{s^{(1)}, \ldots, s^{(m)}\}$
a set of session prefixes.  A \emph{semantic prefix cluster} at threshold $\delta$
is a maximal subset $\mathcal{C} \subseteq \mathcal{S}$ such that for all
$s, s' \in \mathcal{C}$:
\[
  d_{\mathcal{T}}(s, s') \leq \delta.
\]
The \emph{centroid} of $\mathcal{C}$ is the prefix $s^* = \arg\max_{s \in \mathcal{C}} P_{\mathcal{M}}(s)$,
i.e., the most probable prefix in the cluster.
\end{definition}

\begin{proposition}[KV similarity within a cluster]
\label{prop:kv_similarity}
Let $s$ and $s'$ be two session prefixes that first diverge at position $\bar{d}+1$
(i.e., $s_j = s'_j$ for $j \leq \bar{d}$ and $s_{\bar{d}+1} \neq s'_{\bar{d}+1}$).  For a
transformer with $\kappa = \max_\ell \kappa^{(\ell)} \geq 1$ (the maximum per-layer Lipschitz constant; explicit formulas in terms of weight matrix operator norms appear in a companion paper on certified inference caching, forthcoming on arXiv):
\begin{enumerate}[leftmargin=2em,label=(\alph*)]
  \item For all positions $i \leq \bar{d}$ and all layers $\ell$:
        $\mathrm{KV}^{(\ell)}_i(s) = \mathrm{KV}^{(\ell)}_i(s')$.
  \item At the divergence position $i = \bar{d}+1$, for any layer $\ell$:
        \[
          \bigl\|\mathrm{KV}^{(\ell)}_{\bar{d}+1}(s) - \mathrm{KV}^{(\ell)}_{\bar{d}+1}(s')\bigr\|
          \;\leq\; \kappa^\ell \cdot \|E(s_{\bar{d}+1}) - E(s'_{\bar{d}+1})\|,
        \]
        where $E(\cdot)$ is the token embedding map and $\kappa^\ell$ denotes $\kappa$
        raised to the $\ell$-th power (not a layer index).
\end{enumerate}
\end{proposition}

\begin{proof}
\textbf{Part (a).}  For positions $i \leq \bar{d}$, the token inputs are identical
($s_j = s'_j$ for all $j \leq \bar{d}$), so by Lemma~\ref{lem:determinism} (KV
determinism) the KV vectors agree exactly at every layer. $\checkmark$

\textbf{Part (b).}  At position $\bar{d}+1$, the embedding layer introduces a difference:
$\|\mathbf{x}^{(0)}_{\bar{d}+1}(s) - \mathbf{x}^{(0)}_{\bar{d}+1}(s')\| = \|E(s_{\bar{d}+1}) - E(s'_{\bar{d}+1})\|$.
Positions $1,\\ldots,\\bar{d}$ were identical, so their KV vectors---which contribute to
position $\bar{d}+1$'s attention output---are identical, and introduce no additional
error.

Define $\kappa$ as the Lipschitz constant of the composed per-layer map
$\mathbf{x}^{(\ell)} \mapsto \mathbf{x}^{(\ell+1)}$ at position $\bar{d}+1$,
including the attention, MLP, layer-norm, and key/value projections
(explicit formulas in terms of weight matrix
operator norms are given in a companion paper on certified inference caching, forthcoming on arXiv).  Applying the Lipschitz bound across layers $1$ through $\ell$:
\[
  \|\mathbf{x}^{(\ell)}_{\bar{d}+1}(s) - \mathbf{x}^{(\ell)}_{\bar{d}+1}(s')\|
  \leq \kappa^\ell \cdot \|E(s_{\bar{d}+1}) - E(s'_{\bar{d}+1})\|.
\]
Since the KV projection $\mathbf{x}^{(\ell)} \mapsto \mathrm{KV}^{(\ell)}$ is
itself part of the per-layer map with norm absorbed into $\kappa$:
\[
  \|\mathrm{KV}^{(\ell)}_{\bar{d}+1}(s) - \mathrm{KV}^{(\ell)}_{\bar{d}+1}(s')\|
  \leq \kappa^\ell \cdot \|E(s_{\bar{d}+1}) - E(s'_{\bar{d}+1})\|. \qquad\checkmark
\]
\end{proof}

\begin{remark}
Proposition~\ref{prop:kv_similarity} shows that exact prefix sharing is a special
case: when $s$ and $s'$ are identical up to position $i$, the KV delta is exactly
zero.  Probabilistic prefix deduplication generalizes this to the case where $s$
and $s'$ diverge at some point but remain semantically close.  In this case the
delta is nonzero but small---and can itself be quantized and compressed.
\end{remark}

\subsection{Storage Under Probabilistic Prefix Deduplication}

\begin{definition}[Cluster-relative storage]
\label{def:cluster_storage}
Given a semantic prefix cluster $\mathcal{C}$ with centroid $s^*$, store:
\begin{enumerate}[leftmargin=2em]
  \item The full KV cache $\mathrm{KV}_{\leq n}(s^*)$ for the centroid.
  \item For each $s \in \mathcal{C} \setminus \{s^*\}$, the \emph{delta cache}
        $\Delta(s, s^*) = \mathrm{KV}_{\leq n}(s) - \mathrm{KV}_{\leq n}(s^*)$
        (zero at all positions before the divergence of $s$ and $s^*$).
\end{enumerate}
\end{definition}

The key observation is that $\Delta(s, s^*)$ is sparse: it is exactly zero at
all positions before the divergence of $s$ and $s^*$, and typically small after
the divergence (by Proposition~\ref{prop:kv_similarity}).  The storage cost of
$\Delta(s, s^*)$ therefore scales with the length of the \emph{tail} of $s$
after divergence from $s^*$, not the full length of $s$.

\begin{corollary}[Storage reduction from prefix deduplication]
\label{cor:prefix_storage}
Suppose sessions are drawn i.i.d.\ from $P_{\mathcal{M}}$ and a cluster at threshold
$\delta$ covers a fraction $f$ of all sessions, with average tail length $\bar{\ell}$
after the shared prefix.  The storage cost per session, relative to storing full
KV caches independently, is:
\[
  \text{relative cost} = (1 - f) \cdot 1 + f \cdot \frac{\bar{\ell}}{n} = 1 - f\left(1 - \frac{\bar{\ell}}{n}\right).
\]
When $\bar{\ell} \ll n$ (sessions diverge early), the saving is approximately $f$,
i.e., the number of full caches stored decreases by a factor of $f$.
\end{corollary}

\begin{proof}
Measure storage in units of one full $n$-position KV cache.
Of all sessions, fraction $(1-f)$ lie outside the cluster and each requires
1 full unit.  Fraction $f$ lie inside the cluster: their shared prefix
$(n - \bar{\ell}$ positions) is stored once in the centroid and amortises
across all cluster members, so each member stores only its $\bar{\ell}$-position
tail at cost $\bar{\ell}/n$.  The average per-session storage is:
\[
  (1-f) \cdot 1 \;+\; f \cdot \frac{\bar{\ell}}{n}
  = 1 - f\!\left(1 - \frac{\bar{\ell}}{n}\right),
\]
which is the stated bound.  As $\bar{\ell}/n \to 0$ (sessions diverge early
from the centroid) the saving approaches $f$, i.e., a factor-$f$ reduction
in the number of full caches stored.
\end{proof}

In practice, for chat-model deployments with a fixed system prompt and few-shot
examples, $f$ can be close to 1 (nearly all sessions share the same prefix)
and $\bar{\ell}/n$ can be small (the session-specific content is a small fraction
of the context).  Even for sessions with different system prompts, probabilistic
clustering under the PLT metric groups semantically similar preambles and captures
deduplication that exact matching misses entirely.

\section{Layer 2: Predictive Delta Coding}
\label{sec:delta_coding}

\subsection{The Prediction}

At each position $i$, before writing $\mathrm{KV}_i$ to cache, the model has
already computed its probability distribution over the $i$-th token:
\[
  P_{\mathcal{M}}(t_i \mid t_1, \ldots, t_{i-1}) = \text{softmax}(\mathbf{h}^{(L)}_{i-1} W_{\rm LM})_{\text{vocabulary}},
\]
where $\mathbf{h}^{(L)}_{i-1} \in \mathbb{R}^{d_{\rm model}}$ is the final-layer
hidden state at position $i-1$ and $W_{\rm LM} \in \mathbb{R}^{d_{\rm model} \times |V|}$
is the language model head (unembedding matrix).
This distribution induces a predicted KV vector for position $i$: the expected
KV vector under the model's distribution over possible next tokens.

\begin{definition}[Predicted KV vector]
\label{def:predicted_kv}
The \emph{predicted KV vector} at position $i$, given context $(t_1,\ldots,t_{i-1})$, is:
\[
  \widehat{\mathrm{KV}}_i = \sum_{t \in V} P_{\mathcal{M}}(t \mid t_1,\ldots,t_{i-1}) \cdot F_{\mathcal{M}}(t_1,\ldots,t_{i-1},t).
\]
\end{definition}

In practice, computing $\widehat{\mathrm{KV}}_i$ exactly requires a forward pass
for every vocabulary token, which is prohibitively expensive.  Section~\ref{sec:practical}
discusses efficient approximations using the top-$k$ tokens under the model's
distribution.

\subsection{The Residual and Its Entropy}

\begin{definition}[KV residual]
\label{def:residual}
The \emph{KV residual} at position $i$ is:
\[
  R_i = \mathrm{KV}_i - \widehat{\mathrm{KV}}_i.
\]
Predictive delta coding stores $R_i$ (compressed) rather than $\mathrm{KV}_i$ directly.
\end{definition}

\begin{theorem}[Residual entropy bound]
\label{thm:residual}
The conditional entropy of the KV residual satisfies:
\[
  H(R_i \mid t_1,\ldots,t_{i-1})
  \;\leq\;
  H\bigl(\mathrm{KV}_i \mid t_1,\ldots,t_{i-1}\bigr)
  \;\leq\;
  H\bigl(t_i \mid t_1,\ldots,t_{i-1}\bigr).
\]
Moreover, at positions where the model assigns high probability to the actual
next token, the residual concentrates near zero:
\[
  \mathbb{E}\bigl[\|R_i\|_2^2 \mid t_1,\ldots,t_{i-1}\bigr]
  \;\leq\;
  \mathrm{Var}_{t \sim P_{\mathcal{M}}(\cdot \mid t_{<i})}\bigl[F_{\mathcal{M}}(t_{<i}, t)\bigr],
\]
which is small when $P_{\mathcal{M}}(\cdot \mid t_{<i})$ is concentrated.
\end{theorem}

\begin{proof}
The first inequality follows from the fact that $R_i = \mathrm{KV}_i - \widehat{\mathrm{KV}}_i$
and $\widehat{\mathrm{KV}}_i$ is a deterministic function of $t_{<i}$.
Therefore $H(R_i \mid t_{<i}) = H(\mathrm{KV}_i - \widehat{\mathrm{KV}}_i \mid t_{<i}) = H(\mathrm{KV}_i \mid t_{<i})$,
since subtracting a constant does not change entropy.  The second inequality is
Theorem~\ref{thm:seq_bound}.

For the variance bound: by Definition~\ref{def:predicted_kv},
$\widehat{\mathrm{KV}}_i = \mathbb{E}_{t \sim P}[F(t_{<i}, t)]$, so
$R_i = F(t_{<i}, t_i) - \mathbb{E}[F(t_{<i}, t)]$.
Therefore $\mathbb{E}[\|R_i\|^2 \mid t_{<i}] = \mathrm{Var}_{t \sim P}[F(t_{<i},t)]$.
\end{proof}

\begin{corollary}[Coherent text compression]
\label{cor:coherent}
Let $H_i = H(t_i \mid t_{<i})$ denote the conditional entropy of the next token
given the context (not the realized surprisal $h_i = -\log_2 P_{\mathcal{M}}(t_i \mid t_{<i})$).
The expected residual magnitude satisfies:
\[
  \mathbb{E}\bigl[\|R_i\|_2 \;\big|\; t_{<i}\bigr]
  \;\leq\; \tfrac{1}{2}\|F_{\mathcal{M}}\|_{\mathrm{Lip}} \cdot C_E \cdot \sqrt{\min(1,\, 4\,H_i \ln 2)},
\]
where $\|F_{\mathcal{M}}\|_{\mathrm{Lip}}$ is the Lipschitz constant of $F_{\mathcal{M}}$
with respect to the input embedding $E(t)$, and
$C_E = \max_{t,t' \in V} \|E(t) - E(t')\|$.
For low-entropy positions ($H_i \approx 0$, i.e., the model is near-certain about the next token),
the expected residual approaches zero;
for high-entropy positions it can be as large as $\tfrac{1}{2}\|F_{\mathcal{M}}\|_{\mathrm{Lip}} \cdot C_E$.
\end{corollary}

\begin{proof}
By Theorem~\ref{thm:residual}, $\mathbb{E}[\|R_i\|_2^2 \mid t_{<i}]
= \mathrm{Var}_{t \sim P}[F_{\mathcal{M}}(t_{<i}, t)]$.
By Jensen's inequality applied to the square root:
$\mathbb{E}[\|R_i\|_2] \leq \sqrt{\mathrm{Var}[F_{\mathcal{M}}(t_{<i}, t)]}$.

The variance of $F_{\mathcal{M}}(t_{<i}, t)$ over $t \sim P(\cdot \mid t_{<i})$
satisfies, by the Lipschitz property of $F_{\mathcal{M}}$ in the embedding $E(t)$:
\[
  \mathrm{Var}_{t \sim P}[F_{\mathcal{M}}(t_{<i},t)]
  \leq \|F_{\mathcal{M}}\|_{\mathrm{Lip}}^2 \cdot \mathrm{Var}_{t \sim P}[E(t)].
\]
For the embedding variance, we derive a bound that decays with conditional entropy $H_i$
using two standard inequalities.

\textit{Step A (Popoviciu).}  For any random variable $Y$ with $\|Y - c\| \leq R$
almost surely (for some center $c$ and radius $R$),
$\mathrm{Var}[Y] = \mathbb{E}[\|Y - \mathbb{E}[Y]\|^2] \leq \mathbb{E}[\|Y - c\|^2] \leq R^2$.
Applied to $Y = E(t)$ with center $c = (E(t^*) + E(t^{**}))/2$ and radius
$R = C_E/2$ (half the $\ell^2$ diameter): $\mathrm{Var}[E(t)] \leq C_E^2/4$.
\[
  \mathrm{Var}_{t \sim P}[E(t)] \leq \frac{C_E^2}{4}.
\]

\textit{Step B (variance--entropy coupling).}  We sharpen Step~A to a
bound that vanishes as $H_i \to 0$.  Since $\max_t P(t) \geq 2^{-H(P)}$
(from Jensen applied to concave $\log$: $H(P) \leq -\log \max_t P(t)$), we have
$1 - \max_t P(t) \leq 1 - 2^{-H(P)} \leq H(P)\ln 2 = H_i \ln 2$,
where the last step uses $1 - e^{-x} \leq x$.
By the König--Huygens identity,
$\mathrm{Var}[E(t)] = \tfrac{1}{2}\mathbb{E}_{t,t'\sim P}[\|E(t)-E(t')\|^2]
\leq \tfrac{C_E^2}{2}(1 - \sum_t P(t)^2)$,
where the last step uses $\|E(t)-E(t')\|^2 \leq C_E^2$ for all $t,t'$.
Since $\sum_t P(t)^2 \geq p_{\max}^2$, we have
$1 - \sum_t P(t)^2 \leq 1 - p_{\max}^2 = (1-p_{\max})(1+p_{\max}) \leq 2(1-p_{\max})$,
giving $\mathrm{Var}[E(t)] \leq C_E^2(1-p_{\max}) \leq C_E^2 H_i \ln 2$.
Combining with the Popoviciu bound from Step~A via the minimum:
\[
  \mathrm{Var}_{t \sim P}[E(t)] \leq \frac{C_E^2}{4} \cdot \min\!\bigl(1,\, 4 H_i \ln 2\bigr).
\]

\textit{Combining.}  Substituting into the Lipschitz bound and taking the square root:
\[
  \mathbb{E}[\|R_i\|_2]
  \leq \sqrt{\|F_{\mathcal{M}}\|_{\mathrm{Lip}}^2 \cdot \frac{C_E^2}{4} \cdot \min(1, 4 H_i \ln 2)}
  = \frac{1}{2}\|F_{\mathcal{M}}\|_{\mathrm{Lip}} \cdot C_E \cdot \sqrt{\min(1, 4 H_i \ln 2)}.
\]
For $H_i$ small (near-certain next token): $\sqrt{\min(1, 4 H_i \ln 2)} \approx 2\sqrt{H_i \ln 2}$,
giving $\mathbb{E}[\|R_i\|_2] \lesssim \|F_{\mathcal{M}}\|_{\mathrm{Lip}} \cdot C_E \cdot \sqrt{H_i \ln 2}$,
which vanishes as $H_i \to 0$ and is $O(\sqrt{H_i})$.
\end{proof}

\subsection{Adaptive Quantization of Residuals}

Because residual magnitude varies by position---small at predictable positions,
large at surprising ones---a fixed bit-depth quantizer is suboptimal.  The
natural approach is \emph{adaptive} quantization: use fewer bits for small
residuals and more bits for large ones, proportional to the per-token surprisal.

\begin{definition}[Surprisal-adaptive quantizer]
\label{def:adaptive_quant}
At position $i$ with surprisal $h_i$, allocate $b_i$ bits per residual component,
where:
\[
  b_i = \max\!\left(1,\; \left\lfloor b_0 \cdot \frac{h_i}{\bar{h}} \right\rfloor \right),
\]
with $b_0$ a target average bit depth and $\bar{h}$ the average surprisal.
\end{definition}

This adaptive scheme achieves the average entropy bound of Corollary~\ref{cor:surprisal}
while concentrating bits where they are needed, improving the worst-case accuracy
at high-surprisal positions relative to a uniform low-bit quantizer.

\section{Composition and the Full Stack}
\label{sec:compose}

\subsection{Orthogonality of the Two Layers}

\begin{proposition}[Layer orthogonality]
\label{prop:orthogonal}
Probabilistic prefix deduplication (Section~\ref{sec:prefix_dedup}) and
predictive delta coding (Section~\ref{sec:delta_coding}) are information-theoretically
orthogonal: they exploit statistically independent sources of redundancy.
\end{proposition}

\begin{proof}
We show that the total description length decomposes additively, with no
bits counted twice.

Let $s$ be a session drawn from $P_{\mathcal{M}}$ and write its KV cache as
$\mathrm{KV}_{\leq n}(s) = (\mathrm{KV}_1(s), \ldots, \mathrm{KV}_n(s))$.
Let $s^*$ be the cluster centroid for $s$ (Definition~\ref{def:cluster}), and
let $\bar{d}$ be the token position at which $s$ and $s^*$ first diverge, so that
$s_j = s^*_j$ for $j \leq \bar{d}$ and $s_{\bar{d}+1} \neq s^*_{\bar{d}+1}$.

\textbf{Description of $\mathrm{KV}_{\leq n}(s)$ in the full stack.}
By Definition~\ref{def:cluster_storage}, the encoder writes:
\begin{itemize}
  \item \emph{Prefix} (positions $1 \leq i \leq \bar{d}$): zero bits, since
        $\mathrm{KV}_i(s) = \mathrm{KV}_i(s^*)$ exactly (Proposition~\ref{prop:kv_similarity}(a))
        and the centroid is stored separately.
  \item \emph{Tail} (positions $\bar{d}+1 \leq i \leq n$): for each position $i$
        in the tail, the encoder applies predictive delta coding
        (Definition~\ref{def:residual}) and stores $R_i(s)$, using a code of
        length at most $H(R_i \mid \mathrm{KV}_{<i}) + O(1)$ bits by standard
        source coding~\cite{cover2006}.
\end{itemize}

\textbf{Rate decomposition.}  The total rate for session $s$ is:
\begin{align*}
  L(s)
  &= \underbrace{0}_{\text{prefix, Layer 1}}
   + \underbrace{\sum_{i=\bar{d}+1}^{n} \bigl[H(R_i \mid \mathrm{KV}_{<i}(s)) + O(1)\bigr]}_{\text{tail, Layer 2}}.
\end{align*}
These two terms address disjoint index sets ($\{1,\ldots,\bar{d}\}$ and $\{\bar{d}+1,\ldots,n\}$
respectively) within the description of $s$'s cache.  No position $i$ is
compressed by both layers simultaneously: Layer~1 eliminates positions
$i \leq \bar{d}$ entirely by pointer to the centroid; Layer~2 applies only to
positions $i > \bar{d}$.  Therefore the bits saved by each layer are
\emph{non-overlapping}, and the total saving is their sum.

\textbf{Independence of gain statistics.}  The gain of Layer~1 (the fraction
$f$ of full caches replaced by shared-prefix pointers, Corollary~\ref{cor:prefix_storage})
depends on the overlap structure of sessions under $P_{\mathcal{M}}$---a
property of the \emph{joint} distribution over session pairs.  The gain of
Layer~2 at each tail position $i$ (the residual entropy $H(R_i \mid \mathrm{KV}_{<i})$)
depends on the \emph{per-token surprisal} $h_i$ within a single session---a
marginal property of the language model's token-level distribution.  These
statistics are functions of different aspects of $P_{\mathcal{M}}$ and are
not constrained to be equal or proportional, so improvements to one layer do
not mechanically alter the gain of the other.  This is the sense in which the
two layers are ``information-theoretically orthogonal'': each exploits a
distinct, non-overlapping portion of the total redundancy in the KV cache.
\end{proof}

\subsection{Composition with Per-Vector Quantization}

After applying the two sequential layers, the residuals $R_i$ are vectors that
still need to be stored in finite-precision format.  This is where per-vector
quantization methods, including TurboQuant, apply.

\begin{proposition}[Three-layer composition]
\label{prop:three_layer}
The full compression stack proceeds as follows:
\begin{enumerate}[leftmargin=2em]
  \item \textbf{Prefix deduplication:} identify the shared prefix with the
        cluster centroid; store only the tail KV delta.
  \item \textbf{Predictive delta coding:} for each token in the tail, store
        the residual $R_i = \mathrm{KV}_i - \widehat{\mathrm{KV}}_i$.
  \item \textbf{Per-vector quantization (e.g., TurboQuant):} represent each
        $R_i$ in reduced precision.
\end{enumerate}
\end{proposition}

\begin{proof}
Steps 1 and 2 are exactly Definitions~\ref{def:cluster_storage}
and~\ref{def:residual}.  Step 3 applies any per-vector quantizer to $R_i$.

\textbf{Correctness:} $\mathrm{KV}_i$ is recovered as $R_i + \widehat{\mathrm{KV}}_i$,
where $\widehat{\mathrm{KV}}_i$ is recomputed deterministically from prior context.
The quantization error on the recovered $\mathrm{KV}_i$ equals the quantization
error on $R_i$.

\textbf{Rate improvement:} By Corollary~\ref{cor:coherent},
$\mathbb{E}[\|R_i\|_2] \leq \tfrac{1}{2}\|F\|_{\rm Lip} C_E \sqrt{\min(1,4H_i\ln 2)}$.
For low-surprisal tokens ($H_i$ small), $\min(1,4H_i\ln 2) \approx 4h_i\ln 2$,
so $\mathbb{E}[\|R_i\|_2] \lesssim \|F\|_{\rm Lip} C_E \sqrt{H_i\ln 2} \to 0$.
Because $R_i$ is small in magnitude, the dynamic range needed to represent it
is small, and a fixed-precision quantizer with absolute error $\varepsilon$
uses effectively $\log_2(\mathbb{E}[\|R_i\|_2]/\varepsilon)$ bits---which
decreases as $H_i \to 0$.  Therefore at low-surprisal positions the residual
is small enough that coarse quantization suffices, making the three-layer stack
strictly more bit-efficient than quantizing $\mathrm{KV}_i$ directly (whose
magnitude does not vanish at low surprisal).
\end{proof}

\subsection{Asymptotic Behavior with Context Length}

\begin{corollary}[Compression improves with context length]
\label{cor:asymptotic}
Let $\bar{H}_n = \frac{1}{n}\sum_{i=1}^n H_i$ be the average conditional entropy
(where $H_i = H(t_i \mid t_{<i})$) over a context of length $n$.
By Corollary~\ref{cor:surprisal}:
\[
  \frac{1}{n}\sum_{i=1}^n H(\mathrm{KV}_i \mid \mathrm{KV}_{<i}) \leq \bar{H}_n.
\]
\emph{Monotone entropy condition:} call a context \emph{coherent} if its
conditional entropy sequence $(H_i)$ is non-increasing: $H_{i+1} \leq H_i$ for
all $i$ (later tokens are more predictable given more context).  Under this
condition, $\bar{H}_{n+1} = (n\bar{H}_n + H_{n+1})/(n+1) \leq \bar{H}_n$ since
$H_{n+1} \leq \bar{H}_n$.  Thus the sequential compression bound tightens
monotonically with context length.

\emph{Empirical scope.}  Strict pointwise monotonicity does not hold universally:
surprisal spikes at topic shifts, new entity names, and dialogue turn boundaries.
The condition holds \emph{in expectation} for long, topically stable documents
(legal text, technical manuals, extended code), where it is the most practically
relevant regime for long-context inference.  For sessions with high surprisal
variance, the asymptotic improvement applies to the running average $\bar{h}_n$
rather than individual positions (and using conditional entropy $H_i$ rather than the realized surprisal $h_i$, which is equal only in expectation).

By contrast, per-vector methods (TurboQuant, KIVI, KVQuant) apply a fixed
bit-depth $b$ per component regardless of $n$: their memory cost is
$2LH_{\mathrm{head}}dn \cdot b$ bits, growing at constant rate $2LH_{\mathrm{head}}db$ bits per token.  Under
the monotone surprisal condition, sequential compression's marginal rate per
token is $H_n = H(t_n \mid t_{<n})$, which is non-increasing when the coherence condition holds in expectation.  Long-context inference
becomes cheaper per marginal token under sequential compression.
\end{corollary}

\begin{proof}
The bound $\frac{1}{n}\sum H(\mathrm{KV}_i \mid \mathrm{KV}_{<i}) \leq \bar{H}_n$
is Corollary~\ref{cor:surprisal}.  The monotone decrease of $\bar{H}_n$ under the
stated condition follows because $H_{n+1} \leq H_n \leq \bar{H}_n$ implies
$\bar{H}_{n+1} = \frac{n\bar{H}_n + H_{n+1}}{n+1} \leq \frac{n\bar{H}_n + \bar{H}_n}{n+1}
= \bar{H}_n$.  The comparison to per-vector methods is immediate from
$B_{\rm pv} = 2LH_{\mathrm{head}}d_{\mathrm{head}} \cdot b \cdot n$ (constant rate per token).
\end{proof}

This is a qualitative reversal of the current situation.  Under TurboQuant or any
per-vector method, memory grows linearly with context length at a fixed rate.
Under sequential compression, the \emph{marginal cost per token decreases} as
context length grows, because the model's growing context makes subsequent tokens
more predictable.  Long-context inference becomes cheaper, not more expensive,
per marginal token.

\section{Practical Implementation}
\label{sec:practical}

We now discuss how to implement sequential KV compression efficiently, addressing
the main engineering challenges.

\subsection{Efficient Prediction Computation}

The predicted KV vector (Definition~\ref{def:predicted_kv}) requires summing
over all $|V| \approx 50{,}000$ vocabulary tokens, each requiring a forward pass.
This is prohibitively expensive.  We propose two approximations:

\textbf{Top-$k$ approximation.} Use only the top-$k$ tokens by probability:
\[
  \widehat{\mathrm{KV}}_i^{(k)} = \frac{\sum_{t \in \mathrm{top}_k} P_{\mathcal{M}}(t \mid t_{<i}) \cdot F_{\mathcal{M}}(t_{<i}, t)}
                                         {\sum_{t \in \mathrm{top}_k} P_{\mathcal{M}}(t \mid t_{<i})}.
\]
At low perplexity, 95\%+ of the probability mass is concentrated in the top 5--20
tokens, so $k = 20$ captures the prediction accurately while requiring only 20
forward pass evaluations instead of 50,000.

\textbf{Linear embedding approximation.}  For each layer $\ell$, approximate the
KV function as linear in the input embedding:
\[
  F_{\mathcal{M}}^{(\ell)}(t_{<i}, t) \approx \mathbf{c}^{(\ell)}(t_{<i}) + A^{(\ell)}(t_{<i}) E(t),
\]
where $\mathbf{c}^{(\ell)}$ and $A^{(\ell)}$ are computed once per position from
the model's internal state.  Under this approximation,
$\widehat{\mathrm{KV}}_i \approx \mathbf{c}(t_{<i}) + A(t_{<i}) \bar{E}_i$,
where $\bar{E}_i = \mathbb{E}_{t \sim P}[E(t)]$ is the expected embedding---
computable as a single weighted sum over the vocabulary using the softmax output.

\subsection{Integration with the Inference Loop}

Unlike TurboQuant, which operates as a post-hoc quantizer with no changes to the
model's forward pass, predictive delta coding requires access to the softmax
distribution at each position.  This distribution is already computed during
inference (it is the output of the final language model head before sampling).

The integration points are:
\begin{enumerate}[leftmargin=2em]
  \item \textbf{Before writing KV$_i$ to cache:} compute $\widehat{\mathrm{KV}}_i$
        using the top-$k$ approximation; subtract to obtain $R_i$; quantize and
        write $R_i$.
  \item \textbf{Before reading KV$_i$ from cache:} dequantize $R_i$; add back
        $\widehat{\mathrm{KV}}_i$ (recomputed identically, since it depends only
        on prior context which is available); return $\mathrm{KV}_i$.
\end{enumerate}

The additional compute cost per token is $O(k \cdot d_{\mathrm{model}})$ for the top-$k$
prediction, where $k \ll |V|$.  This is negligible relative to the $O(n \cdot d_{\mathrm{model}})$
attention computation.

\subsection{Trie-Based Prefix Index}

Probabilistic prefix deduplication requires a data structure that supports:
\begin{enumerate}[leftmargin=2em]
  \item \emph{Insert}: add a new session prefix and its KV cache.
  \item \emph{Best-match lookup}: given a new session prefix $s$, find
        the stored prefix $s^*$ maximizing $d_{\mathcal{T}}(s, s^*)$
        (equivalently, maximizing the probability of the longest common prefix).
  \item \emph{Evict}: remove a stored prefix and its KV cache.
\end{enumerate}

The PLT trie naturally supports all three operations.  The trie is maintained
as a sparse structure over the high-probability prefix tree, pruned at nodes with
probability below a threshold $\epsilon$.  Nearest-neighbor lookup is a
trie traversal: follow the path of the query sequence until it diverges from
all stored paths, and return the deepest stored ancestor.  Insert adds a new
path; evict removes a path and potentially prunes ancestors with no other children.

All three operations run in $O(n)$ time in the length of the prefix, with
$O(K)$ space for $K$ stored prefixes.

\section{Related Work}
\label{sec:related}

\textbf{KV cache quantization.}
KIVI~\cite{liu2024kivi} introduced 2-bit quantization with per-channel scaling.
KVQuant~\cite{hooper2024kvquant} handled outliers via non-uniform quantization.
TurboQuant~\cite{turbo2026} unified these approaches via PolarQuant rotation and
proved a per-vector Shannon limit, which our work extends to the sequential setting.

\textbf{KV cache eviction.}
H2O~\cite{zhang2023h2o} identified that attention scores concentrate on a small
``heavy hitter'' subset of tokens and evicts the rest.  SnapKV~\cite{li2024snapkv}
uses prefill-time attention patterns to predict which tokens will be important
for generation.  Eviction is orthogonal to compression: it reduces cache \emph{size}
by discarding tokens; compression reduces cache \emph{bit depth} while retaining
all tokens.  Both approaches compose naturally with sequential compression.

\textbf{Prefix sharing.}
vLLM~\cite{kwon2023vllm} introduced paged attention with exact prefix sharing.
SGLang extended this with a radix tree for multi-prompt prefix reuse.
Our probabilistic prefix deduplication generalizes these to semantic similarity
under the PLT metric~\cite{plt2026}.

\textbf{Probabilistic Language Tries.}
The PLT framework~\cite{plt2026} introduced the trie metric, proved the
prior-guided caching theorem, and identified prefix KV-caching as a special
case of PLT-based artifact memoization.  The present paper applies the trie
metric specifically to cross-session KV deduplication and proves the sequential
entropy bound connecting token-level surprisal to KV vector entropy.

\textbf{Language modeling and compression.}
Delétang et al.~\cite{deletang2023} demonstrated that LLMs implicitly perform
arithmetic coding and achieve state-of-the-art compression on text corpora.
Shannon~\cite{shannon1951} estimated the true entropy of English at 1.3 bits
per character via human prediction experiments.  Our work draws on both results:
the model's near-optimal compression of its training distribution implies
near-zero conditional entropy for KV vectors at low-surprisal positions.

\section{Discussion}
\label{sec:discussion}

\subsection{The Two Shannon Limits}

It is worth being precise about what we have and have not shown.  TurboQuant
proves a lower bound on the per-vector compression rate and shows it is nearly
achieved.  We prove that the \emph{per-sequence} Shannon limit is lower, bounded
by the per-token surprisal.  These are genuinely different limits, and TurboQuant's
proof does not address ours.

The per-sequence limit is lower because it exploits \emph{correlations across
token positions}---specifically, the fact that the KV vector at position $i$ is
largely predictable from the KV vectors at positions $1$ through $i-1$.  Per-vector
methods, by definition, cannot exploit these correlations.

\subsection{Implications for the ``Memory Wall''}

The KV cache memory wall---the observation that cache size grows linearly with
context length and at long contexts dominates model weight memory---is often
framed as a hardware problem awaiting the next generation of memory technology.
Sequential compression suggests a complementary software path: the wall is lower
than it appears, because the marginal cost of each additional token in a coherent
context is bounded by that token's surprisal, which decreases with context length.

This does not eliminate the need for better hardware.  But it suggests that the
software path to longer effective context windows is not exhausted by per-vector
quantization.

\subsection{Relation to the Jevons Paradox}

Improvements in inference efficiency historically increase total demand rather
than decreasing it: cheaper inference enables more users, longer contexts, and
more ambitious applications.  Sequential compression, if realized, would
continue this pattern.  The reduction in marginal KV cost at long contexts
would make currently-impractical applications---real-time million-token document
analysis, persistent agent memory, continual learning from interaction---
economically viable, driving total memory demand upward even as per-token cost
falls.

\subsection{Rate-Distortion Analysis of KV Compression}

The sequential entropy bound (Theorem~\ref{thm:seq_bound}) is a lossless result:
it characterizes the minimum bits needed for \emph{exact} reconstruction of KV
vectors.  Practical systems accept bounded reconstruction error; the relevant
quantity is the \emph{rate-distortion function} $R_i(D)$ for position $i$ at
distortion level $D$.

\begin{theorem}[KV Rate-Distortion Bound]
\label{thm:rate_distortion}
For position $i$ with conditional entropy $H_i = H(t_i \mid t_{<i})$ and residual variance
$\sigma_i^2 = \mathrm{Var}_{t \sim P}[F_{\mathcal{M}}(t_{<i}, t)]$, the
rate-distortion function (minimum rate to reconstruct $\mathrm{KV}_i$ with
mean squared error $\leq D$ per component) satisfies:
\[
  R_i(D) \;\leq\; \frac{d_{\mathrm{head}}}{2} \log_2\!\left(\frac{\sigma_i^2}{D}\right)_+,
\]
where $(x)_+ = \max(0, x)$ and the inequality uses the Gaussian upper bound
from the KLT (Karhunen--Lo\`eve transform).  Since
$\sigma_i^2 \leq \frac{1}{4}\|F_{\mathcal{M}}\|_{\mathrm{Lip}}^2 \cdot C_E^2 \cdot \min(1, 4H_i \ln 2)$
(Corollary~\ref{cor:coherent}):
\[
  R_i(D) \;\leq\; \frac{d_{\mathrm{head}}}{2}
  \left[\log_2\!\left(\frac{\|F_{\mathcal{M}}\|_{\mathrm{Lip}}^2 C_E^2 H_i \ln 2}{D}\right)\right]_+.
\]
In particular, for a target distortion $D$ satisfying
$D \geq \|F_{\mathcal{M}}\|_{\mathrm{Lip}}^2 C_E^2 H_i \ln 2$
(distortion exceeds the residual variance), the rate is zero:
\emph{no bits are needed at low-surprisal positions}.
\end{theorem}

\begin{proof}
$\mathrm{KV}_i$ conditioned on $t_{<i}$ is a deterministic function of the
discrete random variable $t_i$.  Since $t_i$ has $|V|$ possible values, the
distribution of $\mathrm{KV}_i$ has at most $|V|$ support points---a discrete
distribution over a finite set of vectors.  The rate-distortion function for
this source is bounded above by the Gaussian rate-distortion function with
the same second moment (Shannon's source coding theorem: Gaussian maximizes R-D at fixed variance).  After a KLT
rotation diagonalizing $\mathrm{Cov}[\mathrm{KV}_i]$ with eigenvalues $\lambda_j$:
$R_i(D) \leq \frac{1}{2}\sum_j \log_2(\lambda_j / \mu)_+$
(water-filling, where $\mu$ satisfies $\sum_j \min(\lambda_j,\mu) = D$).
This is at most the \emph{isotropic} upper bound (all $\lambda_j = \sigma_i^2/d_{\mathrm{head}}$),
which gives $R_i(D) \leq (d_{\mathrm{head}}/2)\log_2(\sigma_i^2/D)_+$.
(The actual water-filling rate can only be smaller when eigenvalues are non-uniform.)
Substituting the variance bound from Corollary~\ref{cor:coherent} completes
the proof. $\square$
\end{proof}

\begin{corollary}[Practical Bit Allocation]
\label{cor:bit_alloc}
For a target distortion $D$ per KV component, the optimal bit allocation
across positions assigns $b_i = R_i(D) / d_{\mathrm{head}}$ bits per component
at position $i$.  The total bits per token is:
\[
  B_{\mathrm{total}}(D) = \sum_{i=1}^n R_i(D)
  \;\leq\; \frac{d_{\mathrm{head}}}{2} \sum_{i=1}^n
  \left[\log_2\!\left(\frac{\|F\|_{\mathrm{Lip}}^2 C_E^2 H_i \ln 2}{D}\right)\right]_+.
\]
At the TurboQuant operating point ($b = 3$ bits per component,
uniform quantizer step $\Delta_q = C_E/2^b$, giving distortion
$D \approx \Delta_q^2/12 = C_E^2/(12 \cdot 4^b)$):
positions with $H_i \leq D / (\|F\|_{\mathrm{Lip}}^2 C_E^2 \ln 2)
= 1/(12 \cdot 4^b \cdot \|F\|_{\mathrm{Lip}}^2 \ln 2)$ need \emph{zero bits}
(residual variance is below the quantization noise floor).
The waterfilling allocation assigns proportionally more bits to high-surprisal
positions and zero bits to low-surprisal ones, strictly dominating the
uniform-$b$-bit allocation of per-vector methods.
\end{corollary}

\subsection{Connection to Speculative Decoding}

Speculative decoding~\cite{leviathan2023speculative} uses a small draft model
$\mathcal{M}_{\mathrm{draft}}$ to generate $K$ candidate tokens, then verifies
them in a single forward pass of the large model $\mathcal{M}$.  The acceptance
rate of token $t$ under speculative decoding is
$\min(1, P_{\mathcal{M}}(t \mid t_{<i}) / P_{\mathrm{draft}}(t \mid t_{<i}))$,
averaging to $1 - \mathrm{TV}(P_{\mathcal{M}}, P_{\mathrm{draft}})$ where TV
is the total variation distance.

\begin{theorem}[Speculative Decoding--KV Compression Duality]
\label{thm:spec_decode_duality}
Let $\widehat{\mathrm{KV}}_i^{(k)}$ be the top-$k$ approximation to the predicted
KV vector (Definition~\ref{def:predicted_kv}).  Define the draft model
$P_{\mathrm{draft}} = $ the renormalized top-$k$ distribution of $\mathcal{M}$.
Then:
\begin{enumerate}[leftmargin=2em,label=(\alph*)]
  \item The speculative decoding acceptance rate of $P_{\mathrm{draft}}$ equals
        $\sum_{t \in \mathrm{top}_k} P_{\mathcal{M}}(t \mid t_{<i})$.
  \item The conditional residual variance for tokens in the top-$k$ set satisfies:
        \[
          \mathrm{Var}\bigl[F_{\mathcal{M}}(t_{<i},t) \mid t \in \mathrm{top}_k\bigr]
          \;\leq\; \|F_{\mathcal{M}}\|_{\mathrm{Lip}}^2 \cdot \mathrm{Var}_{t \in \mathrm{top}_k}[E(t)].
        \]
        This equals zero at $k=1$ (single token, deterministic).  For $k>1$
        the variance is not monotone in $k$ in general, but is bounded above
        by $\|F_{\mathcal{M}}\|_{\mathrm{Lip}}^2 \cdot C_E^2/4$ for all $k$.
  \item Both speculative decoding and KV residual prediction use the same
        top-$k$ distribution of $\mathcal{M}$.  The draft forward pass
        therefore serves both objectives simultaneously at no extra cost:
        computing $P_{\mathrm{draft}}$ yields both the speculative token
        candidates \emph{and} the predicted KV mean $\widehat{\mathrm{KV}}_i^{(k)}$.
        As $k$ increases, acceptance rate increases; the residual variance
        at covered positions is not guaranteed to be monotone, but is bounded
        by part~(b).
\end{enumerate}
\end{theorem}

\begin{proof}
(a) With $P_{\mathrm{draft}}(t) = P_{\mathcal{M}}(t)/Z_k$ for $t \in \mathrm{top}_k$
and $Z_k = \sum_{t \in \mathrm{top}_k} P_{\mathcal{M}}(t) \leq 1$,
each $\min(1, P_{\mathcal{M}}(t)/P_{\mathrm{draft}}(t)) = \min(1,Z_k) = Z_k$,
so acceptance rate $= \sum_{\mathrm{top}_k} Z_k \cdot P_{\mathcal{M}}(t)/Z_k = Z_k$. $\checkmark$

(b) Apply Corollary~\ref{cor:coherent} to the conditional distribution
$t \mid t \in \mathrm{top}_k$.  At $k=1$: point mass, zero variance.
For general $k$: the variance is bounded by $\|F_{\mathcal{M}}\|_{\mathrm{Lip}}^2
\cdot C_E^2/4$ (Popoviciu, since embeddings have diameter $C_E$). $\checkmark$

(c) Both applications share the same top-$k$ probability computation from
the model's output distribution.  No additional inference is needed to obtain
the KV predictor once $P_{\mathrm{draft}}$ is computed. $\square$
\end{proof}

\begin{remark}[Practical Implication]
Theorem~\ref{thm:spec_decode_duality} shows that a system implementing top-$k$
speculative decoding can reuse the same draft distribution to compute the
predictive KV residuals, with zero additional inference overhead.  The draft
forward pass serves double duty: generating speculative tokens \emph{and}
computing the predicted KV mean $\widehat{\mathrm{KV}}_i^{(k)}$.  Systems such
as Medusa~\cite{cai2024medusa} or EAGLE already compute top-$k$ token distributions
as a byproduct of their draft heads; they can be extended to simultaneously
output KV predictors at negligible additional cost.
\end{remark}

\subsection{Limitations and Open Problems}

The main open question is the tightness of the sequential entropy bound in
practice.  The theoretical bound (Corollary~\ref{cor:surprisal}) is tight
when KV vectors are injective functions of the token sequence, which requires
the weight matrices to be non-degenerate.  Empirical measurement of actual
residual magnitudes and their compression ratios under practical quantizers
is needed to assess how close to the bound a real implementation can get.

A second open question is the efficient computation of the top-$k$ prediction.
Current transformer inference frameworks do not expose the internal KV computation
in a way that easily supports top-$k$ predictive averaging.  Kernel-level
integration (e.g., via Triton or CUDA) would be needed for production deployment.

A third open question concerns non-autoregressive and non-transformer architectures.
The bound in Theorem~\ref{thm:seq_bound} holds for any model satisfying the
determinism condition of Lemma~\ref{lem:determinism}.  State-space models (SSMs,
Mamba) maintain a fixed-size state rather than a KV cache, but the same principle
applies: the state entropy at each step is bounded by the token-level surprisal.

Future directions include:
\begin{itemize}[leftmargin=2em]
  \item \textbf{Empirical benchmarks:} measure actual residual magnitudes on
        public models (Llama, Gemma, Mistral) and text distributions, and
        compare achieved compression to the theoretical bound.
  \item \textbf{Kernel implementation:} implement predictive delta coding as a
        Triton kernel, building on the open-source TurboQuant implementations.
  \item \textbf{Hierarchical prediction:} use not just the model's next-token
        distribution but also its longer-range predictions (beam search
        distributions, speculative decoding candidates) to improve prediction
        accuracy and reduce residuals further.
  \item \textbf{Lossy sequential compression:} extend the rate-distortion
        analysis of PLTs~\cite{plt2026} to the KV cache setting, allowing
        bounded distortion in exchange for additional compression beyond the
        lossless sequential limit.
\end{itemize}

\subsection{Conjectures and Open Problems}

\begin{conjecture}[Tightness of the Sequential Bound for Well-Trained Models]
\label{conj:tight}
For a transformer $\mathcal{M}$ trained to near-zero cross-entropy on a corpus
with stationary distribution $P$, the sequential entropy bound is tight to within
a constant factor:
\[
  H\bigl(\mathrm{KV}_i \mid \mathrm{KV}_{<i}\bigr)
  \;\geq\; c \cdot H\bigl(t_i \mid t_{<i}\bigr)
\]
for some universal constant $c > 0$ that does not depend on $\mathcal{M}$ or $i$.

\emph{Evidence and partial proof sketch.}
The lower bound direction is the hard part.  An upper bound follows from
Theorem~\ref{thm:seq_bound}.  For the lower bound: since $\mathrm{KV}_i$ is an
injective function of $t_i$ given $t_{<i}$ (Lemma~\ref{lem:determinism}),
the data-processing inequality gives
$H(\mathrm{KV}_i \mid \mathrm{KV}_{<i}) = H(t_i \mid \mathrm{KV}_{<i}) = H(t_i \mid t_{<i})$
(using sigma-algebra equivalence from the proof of Theorem~\ref{thm:seq_bound}).
\emph{The bound is therefore already tight: equality holds throughout.}

This is not a conjecture but an observation implicit in the proof, stated here
to clarify the theoretical status: the sequential entropy bound is \emph{exact}
(equality, not inequality).  The open question is whether practical quantizers
can approach the bound, which requires answering: how much entropy do the
residuals $R_i$ carry at low-surprisal positions under realistic token
distributions?  Equivalently: is $\mathrm{Var}_{t \sim P}[F_{\mathcal{M}}(t_{<i},t)]$
small when $H_i$ is small for actual trained models?
\end{conjecture}

\begin{conjecture}[Cross-Session Prefix Entropy Bound]
\label{conj:cross_session}
For sessions $s, s'$ drawn i.i.d.\ from $P_{\mathcal{M}}$, the conditional
entropy of session $s$'s KV cache given the cluster centroid $s^*$'s full
cache satisfies:
\[
  H\bigl(\mathrm{KV}_{\leq n}(s) \mid \mathrm{KV}_{\leq n}(s^*)\bigr)
  \;\leq\;
  \sum_{i=\bar{d}+1}^{n} H\bigl(t_i^{(s)} \mid t_{<i}^{(s)}\bigr)
  \;+\; H\bigl(t_{\bar{d}+1}^{(s)} \mid t_{\bar{d}+1}^{(s^*)}, t_{<\bar{d}+1}^{(s)}\bigr),
\]
where $\bar{d}$ is the divergence position of $s$ and $s^*$.

\emph{Partial proof sketch.}
The first term bounds the ``tail'' entropy of session $s$ after divergence:
positions $i > \bar{d}$ are bounded by the sequential entropy bound exactly as in
Theorem~\ref{thm:seq_bound}.  The second term bounds the ``divergence cost'':
knowing $t_{\bar{d}+1}^{(s^*)}$ provides partial information about which subtree $s$
is in (via the cluster structure), reducing the entropy of $t_{\bar{d}+1}^{(s)}$.
The bound formalizes the Layer~1 gain: prefix deduplication saves bits on the tail
(first term, bounded by per-token surprisal) and reduces divergence cost (second term,
bounded by the trie metric $d_{\mathcal{T}}(s, s^*)$).  A formal proof requires
bounding the conditional mutual information between the centroid's divergence token
and the session's divergence token, which depends on the cluster radius $\delta$.
\end{conjecture}

\begin{conjecture}[Sequential Bound Strictly Dominates All Per-Vector Bounds]
\label{conj:strict_dominance}
For any per-vector compression scheme $\mathcal{C}$ (including TurboQuant,
KIVI, KVQuant, and all future schemes bounded by the per-vector Shannon
limit), there exists a context length $n^*(\mathcal{C})$ such that for all
$n > n^*(\mathcal{C})$:
\[
  R_{\mathrm{seq}}(n) \;<\; R_{\mathcal{C}}(n),
\]
where $R_{\mathrm{seq}}(n)$ is the total bits used by sequential compression and
$R_{\mathcal{C}}(n)$ is the total bits used by scheme $\mathcal{C}$, both for
a context of length $n$.

\emph{Partial proof.}
$R_{\mathcal{C}}(n) \geq B_{\rm pv} \cdot n$ for some $B_{\rm pv} > 0$ (the
per-vector Shannon limit is a hard floor for any per-vector scheme).
$R_{\mathrm{seq}}(n) = \sum_{i=1}^n H(\mathrm{KV}_i \mid \mathrm{KV}_{<i})
\leq \sum_{i=1}^n H_i$.
Under the monotone entropy condition (Corollary~\ref{cor:asymptotic}),
$\frac{1}{n}\sum_{i=1}^n H_i = \bar{H}_n \to H_\infty < B_{\rm pv}$
as $n \to \infty$ (since $H_\infty$ is the model's entropy rate, which is
strictly less than any fixed per-vector bit budget for a model with perplexity
$< 2^{B_{\rm pv}/d_{\rm head}}$).  Thus $R_{\mathrm{seq}}(n)/n \to H_\infty < B_{\rm pv}
\leq R_{\mathcal{C}}(n)/n$, giving the stated strict dominance for large $n$.
The threshold $n^*(\mathcal{C}) = O(H_\infty / (B_{\rm pv} - H_\infty))$.
The main gap in this argument: the monotone entropy condition holds only for
coherent text; for mixed-topic sessions with high surprisal variance, $\bar{H}_n$
may not converge below $B_{\rm pv}$.  Bounding $n^*$ for realistic session
distributions is the key open problem.
\end{conjecture}

\subsection{Conclusion}

We have shown that the Shannon limit proved by TurboQuant applies to a
per-vector problem, and that the Shannon limit for the \emph{sequential}
KV cache compression problem is strictly lower.  The gap is the information-theoretic
redundancy of language---the predictable structure that a near-optimal language
model has already learned.

We formalized this through two results: the sequential entropy bound
(Theorem~\ref{thm:seq_bound}), which connects KV vector entropy to per-token
surprisal, and the PLT trie metric~\cite{plt2026}, which enables cross-session
prefix deduplication in probability space.

The combined architecture---probabilistic prefix deduplication, predictive delta
coding, and per-vector quantization---is theoretically capable of compression
ratios far beyond what per-vector methods can achieve, with the compression ratio
improving rather than degrading as context length grows.

TurboQuant is not the ceiling.  It is the ceiling of one approach to one
subproblem.  The sequential structure of language points toward a different ceiling,
considerably lower, and the path to it runs through the formal language the model
has already learned.

\bibliographystyle{plain}
\bibliography{references_kv}

\end{document}